\documentclass[10pt, a4paper]{article}

\usepackage{times}
\usepackage{soul}
\usepackage{url}
\usepackage[hidelinks]{hyperref}
\usepackage[utf8]{inputenc}
\usepackage[small]{caption}
\usepackage{graphicx}
\usepackage{amsmath}
\usepackage{amsthm}
\usepackage{amssymb}
\usepackage{booktabs}
\usepackage{algorithm}
\usepackage{algorithmic}
\usepackage[switch]{lineno}
\usepackage{aligned-overset}
\usepackage{dsfont}
\usepackage{multirow}
\usepackage{enumerate}
\usepackage{appendix}

\newtheorem{theorem}{Theorem}
\newtheorem{defi}{Definition}
\newtheorem{lem}{Lemma}
\newtheorem{rem}{Remark}

\title{ExLipBaB: Exact Lipschitz Constant Computation \\ for  Piecewise Linear Neural Networks}

\author{
    Tom Splittgerber\\
    University of Bremen \\
    \texttt{tsplittg@uni-bremen.de}
}

\begin{document}

\maketitle

\begin{abstract}
	It has been shown that a neural network's Lipschitz constant can be leveraged to derive robustness guarantees, to improve generalizability via regularization or even to construct invertible networks. Therefore, a number of methods varying in the tightness of their bounds and their computational cost have been developed to approximate the Lipschitz constant for different classes of networks. However, comparatively little research exists on methods for exact computation, which has been shown to be NP-hard. Nonetheless, there are applications where one might readily accept the computational cost of an exact method. These applications could include the benchmarking of new methods or the computation of robustness guarantees for small models on sensitive data. Unfortunately, existing exact algorithms restrict themselves to only ReLU-activated networks, which are known to come with severe downsides in the context of Lipschitz-constrained networks. We therefore propose a generalization of the LipBaB algorithm to compute exact Lipschitz constants for arbitrary piecewise linear neural networks and $p$-norms. With our method, networks may contain traditional activations like ReLU or LeakyReLU, activations like GroupSort or the related MinMax and FullSort, which have been of increasing interest in the context of Lipschitz constrained networks, or even other piecewise linear functions like MaxPool.
\end{abstract}

\section{Introduction}
In recent years, various methods have been developed that require the computation or approximation of a neural network's (NN's) Lipschitz constant. As the Lipschitz constant quantifies a NN's maximum change in output to small changes in input, a core application is the quantification of a NN's robustness to adversarial examples \cite{szegedy_intriguing_2014,cisse_parseval_2017,tsuzuku_lipschitz-margin_2018,weng_towards_2018}. Furthermore, it can be used for regularization \cite{gouk_regularisation_2020} and the construction of invertible networks and Residual Normalizing Flows \cite{behrmann_invertible_2019,chen_residual_2019}. Unfortunately, an exact computation of the Lipschitz constant is NP-hard, which was proven in the case of ReLU networks and the $||.||_2$-norm \cite{virmaux_lipschitz_2018}. The majority of research in the field therefore focuses on scalable methods for increasingly tight approximations of the Lipschitz constant \cite{fazlyab_efficient_2019,weng_evaluating_2018,shi_efficiently_2022,pintore_computable_2024}. Newly proposed methods usually come with certain theoretical guarantees and are also benchmarked on example networks, ideally with a known true Lipschitz constant. To our knowledge, there are currently only two approaches that are able to exactly compute this true constant for non-constructed examples: LipMIP \cite{jordan_exactly_2020} and LipBaB \cite{bhowmick_lipbab_2021}. LipMIP assumes that the $||.||_1$ or $||.||_\infty$ vector norms are used to define the Lipschitz constant, proposing a generalization to norms which they call ``linear''. The LipBaB algorithm works for arbitrary $p$-norms. However, both methods are only applicable to NNs using the ReLU activation. In the context of Lipschitz-constrained NNs specifically, this activation is however seen with increasing scepticicsm as authors have proven inapproximabilty results for layerwise constrained ReLU-networks \cite{huster_limitations_2019} and raised issues of diminishing gradients for deep networks \cite{anil_sorting_2019}.\\
In practical applications in which the Lipschitz constant only needs to be computed once, like robustness analyses of NNs working with sensitive (small) data \cite{pauli_lipschitz_2023}, which would benefit from an exact computation, as well as in benchmarks mimicking such applications, we therefore argue that there is a need for more general exact computation algorithms to obtain best possible benchmarks and robustness bounds. Even in cases where other activation functions could theoretically be expressed as ReLU-networks, this would result in a sharp increase in computational cost due to increased network size and existing methods would nonetheless have to be adapted to residual networks \cite{pauli_novel_2024}.
We therefore propose in this paper a generalization of the LipBaB algorithm \cite{bhowmick_lipbab_2021}, which we call \textbf{Ex}tended \textbf{LipBaB} (ExLiBaB) and which is able to compute the exact local and global Lipschitz constant w.r.t. arbitrary $p$-norms for any continuous piecewise linear NN. ExLipBaB is therefore in particular also applicable to networks using the GroupSort activation or the related FullSort and MinMax, which are of particular interest in the field of Lipschitz constraints \cite{anil_sorting_2019,pauli_novel_2024}. Networks may also include general continuous piecewise linear functions like MaxPool or learned spline activations \cite{tao_piecewise_2022,ducotterd_improving_2024}. Like \cite{bhowmick_lipbab_2021}, we leverage the Branch-and-Bound framework to split the input space of piecewise linear NNs into local regions on which the Lipschitz constant can be easily computed. This framework results in an exponential runtime in the worst case but a faster runtime in most practical applications, as not all possible regions will actually have to be evaluated. In cases where a run of our algorithm to completion would be computationally infeasible, it can be stopped at any time and still deliver a hard upper and lower bound for the Lipschitz constant for a highly flexible class of NNs and arbitrary $p$-norms.

\section{Preliminaries}\label{sec: methods}
We first define the \emph{Lipschitz norm}, which is usually referred to as the \emph{Lipschitz constant} \cite{jordan_exactly_2020}:
\begin{defi}
	The Lipschitz norm of a function \\ $f:(\mathbb{R}^d,||.||_p)\to(\mathbb{R}^m ,||.||_q)$ over an open set $\Omega\subset\mathbb{R}^d$ is defined as:
	\begin{equation*}
		L_{p\to q}(f,\Omega) := \sup_{x_1\neq x_2 \in\Omega} \frac{||f(x_1)-f(x_2)||_q}{||x_1-x_2||_p} \quad .
	\end{equation*}
	If this supremum exists and is finite, $f$ is called Lipschitz continuous on $\Omega$.
\end{defi}
For better readability, we will usually omit the subscripts $p, q$ and $p\to q$.\\
If the set $\Omega$ is equal to the entire $\mathbb{R}^d$, this is often referred to as the \emph{global} Lipschitz constant. In the more general case in which $\Omega$ is an arbitrary open subset, $L(f,\Omega)$ is referred to as the \emph{local} Lipschitz constant \cite{jordan_exactly_2020,bhowmick_lipbab_2021}. This should not be confused with the mathematical concept of local Lipschitz continuity. 

Many approximation approaches, like the basic layerwise approximation \cite{gouk_regularisation_2020} or LipSDP \cite{fazlyab_efficient_2019,pauli_novel_2024} only aim at computing the global Lipschitz constant. LipMIP \cite{jordan_exactly_2020} on the other hand is able to exactly compute the local Lipschitz constant w.r.t ``linear'' norms, whereas the LipBaB approach \cite{bhowmick_lipbab_2021} is able to exactly compute the local and global constant w.r.t. arbitrary $p$-norms; both algorithms are however only applicable to ReLU networks.

For our generalization of the LipBaB approach, we assume that all functions contained in our NN are continuous piecewise linear (PWL) \cite{tao_piecewise_2022,gorokhovik_piecewise_1994}:

\begin{defi}\label{def: methods-pwl}
	Let $\Omega\subset \mathbb{R}^{d_1} $, then a function $\alpha: \Omega\to \mathbb{R}^{d_2}$ is called PWL if there exists a set $\{\mathcal{P}_1,\dots, \mathcal{P}_k\}$ of polyhedra (we will always assume convex polyhedra) with interiors $\mathcal{P}_i^\circ$ such that:
	\begin{enumerate}[i)]
		\item $\bigcup_{i=1}^k \mathcal{P}_i = \Omega$ and $\mathcal{P}_i^\circ \neq\emptyset\; \forall i =1,\dots,k$
		\item $\mathcal{P}_i^\circ \cap \mathcal{P}_j^\circ = \emptyset$ for all $i\neq j$
		\item For all $i=1,\dots,k$, the function $f$ restricted to $\mathcal{P}_i$,  is an affine function, i.e. there exist $T^i\in\mathbb{R}^{d_2\times d_1}, t^i\in\mathbb{R}^{d_2}$ such that $f|_{\mathcal{P}_i}(x) = T^ix+t^i$.
	\end{enumerate}
\end{defi}
In the following sections, we will always assume that we are working with a neural network $f:\mathbb{R}^{d_0}\to \mathbb{R}^{d_L}$ of the form:
\begin{equation}\label{eq: methods-formNN}
	f(x) = \alpha_L \circ \mathfrak{W}_L \circ \alpha_{L-1} \dots \circ \alpha_1 \circ \mathfrak{W}_1(x),
\end{equation}
where each $\mathfrak{W}_i$ is an affine function $\mathfrak{W}_i(x) = W_i(x)+ w_i$ corresponding to a linear layer and each $\alpha_i$ is a continuous PWL function. One can easily see that the assumption that the linear and PWL layers alternate comes without loss of generality, as we can simply either pad layers with the identity function or join layers (concatenations of affine/PWL functions are again affine/PWL). \\
NNs based on PWL activations are commonplace in modern machine learning. An overview over the concept can be found in \cite{tao_piecewise_2022}. We prove in the supplementary material that componentwise PWL activations like ReLU, LeakyReLU or ParametricReLU (with a slope learned during training) as well as GroupSort (and the related FullSort and MinMax) fulfill definition \ref{def: methods-pwl}. A similar argument for the MaxPool function can be found in \cite{gehr_ai_2018}. \\
For the ReLU activation in $d$ dimensions, a full decomposition of the input space into linear regions would have $2^d$ polyhedron elements. For use in our algorithm, we will therefore use an equivalent decomposition, where the activation state is only fixed for a subset of neurons in each region. This transfers the approach of \cite{bhowmick_lipbab_2021} into the general PWL setting:
\begin{rem}\label{rem: methoids-pwl2}
	Let $\Omega\subset \mathbb{R}^{d_1} $, then a function $\alpha: \Omega\to \mathbb{R}^{d_2}$ is called PWL if there exists a set $\{\mathcal{Q}_1, \ldots, \mathcal{Q}_{\tilde{k}}\}$ of polyhedra such that for every neuron $n\in 1,\dots,d_2$, there exists a subset $\mathcal{I}_n\subset\{1,\dots\tilde{k}\}$, such that:
	\begin{enumerate}
		\item $\bigcup_{i\in\mathcal{I}_n} \mathcal{Q}_i = \Omega$
		\item for all $i\in\mathcal{I}_n$, there exist $T^{ni}\in\mathbb{R}^{1\times d_1}, t^{ni}\in\mathbb{R}$, such that $\pi_n \circ f|_{\mathcal{Q}_i}(x) = T^{ni}(x)+t^{ni}$ is an affine function , where $\pi_n$ is the projection onto the $n$-th vector element. 
	\end{enumerate}
\end{rem}

\subsection{Lipschitz constant and Jacobian}

It is well known that the Lipschitz constant of a linear function $A(.)+b:(\mathbb{R}^d, ||.||_p)\to(\mathbb{R}^m, ||.||_q)$ equals the induced matrix norm $||A||_{p\to q}:=\max_{x\neq 0} \frac{||Ax||_q}{||x||_p}$.
Such matrix norms are comparatively easy to compute for common $p$ and $q$ \cite{johnston_how_2016}. Our goal is therefore to reduce the computation of $L(f, \Omega)$ to computing the maximum over all matrix norms in the PWL decomposition of $f$:
\begin{theorem}\label{thm: methods-maxOverPoly}
	Let $\alpha:(\mathbb{R}^{d_1}, ||.||_p)\to (\mathbb{R}^{d_2}, ||.||_q)$ be a PWL function.  Using the notation from definition \ref{def: methods-pwl}, the following holds:
	\begin{equation}\label{eq: methods-maxOverPoly}
		L_{p\to q}(\alpha,\Omega) = \max_{i = 1,\dots, k} ||T_i||_{p\to q} \quad.
	\end{equation}
\end{theorem}
The proof of this theorem is very similar to the arguments used in \cite{bhowmick_lipbab_2021} and therefore not discussed in detail here - we instead refer to the supplementary material. Essentially, we use a result from \cite{jordan_exactly_2020} to relate the Lipschitz constant of $f$ to the supremum over the norm of its generalized (Clarke) Jacobian over $\Omega$. Then, we show that only differentiable points need to be considered for this supremum, which are exactly the points in the interior of polyhedra $\mathcal{P}_i$ in the notation of definition \ref{def: methods-pwl}.

\section{Methodology}	
In theorem \ref{thm: methods-maxOverPoly}, we have reduced the computation of $L(f,\Omega)$ to the discrete optimization problem $\rho_I: \max_{i \in \mathcal{S}} ||T_i||$, with $\mathcal{S} := \{1,\dots, k\}$ the finite index set of all polyhedra $\mathcal{P}_i$ in the PWL decomposition of the NN $f$. Like \cite{bhowmick_lipbab_2021}, we follow the standard branch-and-bound (BnB) framework to recursively solve this problem. A high-level overview of our algorithm is given in listing \ref{alg: simplifiedMain}.\\
At the core of our algorithm is the set $\varrho$ of active subproblems, each representing a discrete optimization problem over a subset $\mathcal{S}_\rho\subseteq \mathcal{S}$. We will impose that the input space associated with each $\rho\in\varrho$ in $\mathbb{R}^{d_0}$ is always a polyhedron. We follow \cite{lee_first_2004}, and describe the three main components of our BnB algorithm:
\begin{enumerate}
	\item \emph{Bounding} $\rho\in\varrho$ with good upper bounds, which is described in section \ref{sec: lipBounds}
	\item \emph{Branching} a subproblem $\rho\in\varrho$ into new subproblems $\rho'$ with $\mathcal{S}_{\rho'}\subseteq \mathcal{S}_\rho$, such that an optimal solution for any $\rho'$ is also an optimal solution to $\rho$, see section \ref{sec: methods-branch}
	\item An efficient method to obtain \emph{lower bounds}, see section \ref{sec: symProp}
\end{enumerate}

\begin{algorithm}[t]
	\caption{Overview ExLipBaB}
	\hspace*{\algorithmicindent} \textbf{Input:}
	Network $f$, input polyhedron $\Omega$, approximation factor $\theta$\\
	\hspace*{\algorithmicindent} \textbf{Output: }
	Tuple (glb, gub); global lower- and upper bound for $L(f, \Omega)$ where gub $\leq \theta\cdot$ glb	
	\begin{algorithmic}[1]
		\STATE initial activation pattern  $\leftarrow $ SymProp($\Omega$, $f$)
		\STATE initial subproblem $\rho_I \leftarrow$ SubProblem($\Omega, \dots$)
		\STATE $\rho_I.L_{ub} \leftarrow$ LipschitzBounds($\rho_I$, $\dots$)
		\STATE initialize glb, gub $\leftarrow (0, \rho_I.L_{ub})$
		\STATE initialize: $\varrho \leftarrow $ maxHeap($[\rho_I]$)
		\IF{$\tilde{l}$ == L}
		\STATE glb $\leftarrow \rho_I.L_{ub}$  \hspace{0.3cm}  //trivial case
		\ENDIF
		\WHILE{gub $> \theta\cdot $ glb}
		\STATE $\rho' \leftarrow \arg\max_{\rho\in\varrho}(\rho.L_{ub})$
		\STATE $\varrho \leftarrow \varrho\backslash\{\rho'\}$
		\STATE (new\_subproblems, glb) $\leftarrow$ Branch($\rho'$, $\dots$)
		\STATE $\varrho \leftarrow (\varrho \; \cup $ new\_subproblems )
		\STATE gub $\leftarrow \max_{\rho\in\varrho} \rho.L_{ub}$
		\ENDWHILE
		\RETURN (glb, gub)
	\end{algorithmic}
	\label{alg: simplifiedMain}	
\end{algorithm}
The set $\varrho$ is initialized as $\varrho= \{\rho_I\}$, where the bounding of $\rho_I$ is describes in section \ref{sec: symProp}. Out algorithm then maintains a global upper bound (gub) and global lower bound (glb) over all subproblems. In the recursive step, the subproblem $\rho\in\varrho$ with the largest upper bound is removed from $\varrho$ and branched into subproblems associated with smaller regions of the input space. For each smaller subproblem $\rho'$, an upper bound is computed and the gub is updated accordingly. Also, if the upper bound is smaller than the glb, we discard $\rho'$ (\emph{fathomed by bounds}). If we can solve $\rho'$ exactly, we also return a lower bound $LB$ and update the glb as $\max(\text{glb}, LB)$; afterwards, we again discard $\rho'$ (\emph{fathomed by optimality}). If a $\rho'$ is not discarded by either of these rules, it is appended into $\varrho$. The algorithm runs until all subproblems are fathomed or until $\gamma\cdot \text{glb} > \text{gub }$ for a pre-determined approximation factor $\gamma$. Either way, termination is guaranteed and if the algorithm is run to completion, the true maximum is found \cite{lee_first_2004,gendron_parallel_1994}.
In the following sections, we describe the components of our algorithm in more detail.

\subsection{Abstract interpretation}\label{sec: abstractInterp}
At creation, each subproblem $\rho$ is assigned an associated polyhedron $\Omega_\rho\subset\Omega$ in input space. We consider $\rho$ exactly solvable, if $f$, constrained to $\Omega_\rho$, is fixed linear, in which case we only have to compute a single matrix norm to solve $\rho$.\\
To determine whether $f$ is fixed linear on $\Omega_\rho$, we compute a list of neurons in $f$ which we call $\star$-neurons, which potentially take different activation states in $\Omega_\rho$. For this, we leverage the concept of abstract interpretation \cite{gehr_ai_2018}. Given a layer index $1\leq l \leq L$, we wish to find an \emph{abstract transformer} $f_l^\#$ that transforms the polyhedron $\Omega_\rho$ of possible inputs into $f$ to a $d_l$-dimensional polyhedron $f_l^\#(\Omega_\rho)$ of potentially possible pre-activation states in the $l$-th layer. We can then intersect $f_l^\#(\Omega_\rho)$ with all polyhedra in the PWL decomposition of $\alpha_l$. If for a neuron $n$, there are at least two polyhedra with non empty intersection with $f_l^\#(\Omega_\rho)$, such that the activation state of $n$ on these polyhedra is different, then we call $n$ a $\star$-neuron.\\
An abstract transformer $f_l^\#$ can simply be constructed by concatenating abstract transformers of the individual layers of $f$. For linear layers, the linear transformation of a polyhedron automatically is a polyhedron. For PWL layers $\alpha_m$, the polyhedron of potential pre-activations is intersected with all polyhedra of the PWL decomposition of $\alpha_m$. Then, each intersection is transformed with the corresponding fixed linear state of $\alpha_m$ and finally a polyhedron is constructed that is a superset of all these smaller transformed polyhedra. This process is described in more detail in \cite{gehr_ai_2018} (note that what the authors call ``CAT'' functions is simply a subclass of PWL functions). 

For each layer, we record the minimum and maximum value of the activation weights and biases for each neuron over all non-empty intersections in an interval matrix $[\Lambda]_l$ and and interval vector $[\lambda]_l$ respectively. In other words, these interval matrices (vectors) contain at each entry an interval of plausible values for the activation state at the given entry (see e.g. \cite{rump_fast_2012}). We also record with $\tilde{l}$ the first layer of $f$ in subproblem $\rho$ which contains $\star$-neurons.

In abstract interpretation, one often restricts the type of abstract sets to be boxes or zonotopes instead of polyhedra. This drastically improves performance at the cost of looser bounds, which is why the initialization algorithm which we describe in section \ref{sec: symProp} uses the box region for all newly added abstract variables. Note that the tightness of the bounds in the initial symbolic propagation only affects initialization and does not alter the exactness of the final result.
\subsection{Initialization}\label{sec: symProp}
To initialize the set of subproblems $\varrho$, an activation pattern and an upper bound have to be computed for the initial subproblem and an initial glb has to be specified. In \cite{bhowmick_lipbab_2021}, the initial glb is simply set to zero and updated only when a subproblem is solved exactly. We instead propose to evaluate the norm of the Jacobian of $f$ in a few randomly chosen points in $\Omega$. This is computationally inexpensive and, because of theorem \ref{thm: methods-maxOverPoly}, results in a hard lower bound if $f$ is differentiable in all chosen points. Theoretically, the probabilty of accidentally choosing a non-differentiable point is zero \cite{jordan_exactly_2020}; in practice, due to floating point inaccuracy, it is only very small. However, the case of accidentally selecting a non-differential point could easily be prevented algorithmically. We test in section \ref{sec: appl-lowerBound} the implications of using a non-trivial lower bound initialization in practice.

For the computation of the initial activation pattern, we adapt the SymProp algorithm from \cite{bhowmick_lipbab_2021} to the general PWL setting. This algorithm makes use of symbolic propagation \cite{li_analyzing_2019} and does not propagate explicit polyhedra though $f$ as was described in section \ref{sec: abstractInterp}. Instead, it propagates linear relations whenever possible. This approach can help to reduce the overapproximation resulting from the well-known dependeny problem. The idea is to represent the state of a neuron as an abstract linear equation of input variables if the activation state of $f$ is fixed linear on that neuron. If the activation state of the neuron isn't fixed, a new abstract input variables is introduced to preserve linear relations in deeper layers. We detail the propagation through the first layer, the other layers follow recursively:

For $f$, we assume the notation of equation \eqref{eq: methods-formNN}. The input is represented as a $d_0$-dimensional symbolic variable that takes values in a polyhedron $\Omega_\rho$. Let $C, c$ be the half space representation of this polyhedron, i.e. $\Omega_\rho = \{x\in\mathbb{R}^{d_0} : Cx \leq c\}$ and define define $\widetilde{B}:= W_1$, $\widetilde{b} := w_1$. Denote also by $A_i, a_i$ the half space constraints of the $i$-th polyhedron $\mathcal{P}_i$ in the PWL decomposition of $\alpha_1$ according to remark \ref{rem: methoids-pwl2}. Formally, all of these variables would need an additional index because they are iterated upon in the recursion. However, to improve readability, we will slightly abuse notation and omit these indices. We then first determine, which activation states $\alpha_1$ may possibly take. For this, we set 
\begin{equation*}
	\widetilde{A}_i:= \text{stack}(A_i \cdot \widetilde{B}, C) \quad,\quad \widetilde{a}_i:= \text{stack}(a_i - A_i\cdot \widetilde{b}, c) \quad,
\end{equation*}
where stack$(.)$ represents the vertical stacking of two matrices.\\
It holds that
\begin{equation}\label{eq: symprop-intersectEquiv}
	\mathfrak{W}_1(\Omega_\rho)\cap \mathcal{P}_i \neq \emptyset \iff \{x\in\mathbb{R}^{d_0}: \widetilde{A}_i\cdot x \leq \widetilde{a}_i\} \neq \emptyset.
\end{equation}
Let $\mathcal{I}$ be the set of indices of the polyhedra that fulfill either side of equation \eqref{eq: symprop-intersectEquiv}. The candidate set of $\star$-neurons in the current layer is then the set of neurons whose state is determined in at least two $\mathcal{P}_i, i\in\mathcal{I}$. In order to reduce the number of unnecessary splits, we however only classify a neuron as $\star$-neuron, if the activation states differ in these polyhedra. Using the notation from remark \ref{rem: methoids-pwl2}, we demand that either $T^{ni}$ or $t^{ni}$ differ between at least two polyhedra. This check is necessary for certain decompositions and less so for others; for example, in the full PWL decomposition of a $d$-dimensional ReLU activation according to definition \ref{def: methods-pwl}, any neuron is fixed active in $2^{d-1}$ polyhedra whereas we describe an alternative decomposition in the supplementary material where it is only fixed active in exactly one polyhedron.

Let $\mathcal{J}$ be the index set of $\star$-neurons at the current layer. For a neuron $n\notin\mathcal{J}$, we denote by $T^{ni}$ and $t^{ni}$ the $n$-th row and $n$-th bias element of the unique activation state of $\alpha$ on $n$ respectively. To propagate the symbolic linear relations on non-$\star$-neurons though $\alpha_1$, we can simply apply these affine functions. For $\star$-neurons however, we can no longer maintain linear relationships. Like \cite{li_analyzing_2019} suggests, we therefore concretize the post-activation bounds of these neurons and add new symbolic variables to the input set that take values within these bounds. Thereby, we hope to maintain linear relations in following layers.\\
Let $j\in\mathcal{J}$ and denote by $\min(j),\max(j)$ the minimal and maximal post-activation value of this $\star$-neuron. We then form $\widehat{\Omega}_\rho := \Omega_\rho \times\big($\LARGE$\times$\normalsize$_{j\in\mathcal{J}} [\min(j),\max(j)]\big)$, which is a polyhedron of dimensionality $\widehat{d}_0 := d_0 + |\mathcal{J}|$. 

Let $T$ be the $d_1\times d_0$ matrix and $t$ be the $d_1$-dimensional vector, whose rows are zero for $\star$-neuron indices and equal to $T^{ni}$, respectively $t^{ni}$ for neurons with a unique activation state. 
We then define a new vector $\widehat{b}$ of length $d_1$ and a matrix $\widehat{B}$ of shape $d_1\times\widehat{d}_0$, by setting $\widehat{b} := T \cdot \widetilde{b} + t$ and, for $n, j\in d_1\times\widehat{d}_0$:
\begin{equation}
	\widehat{b}_{nj} := \begin{cases}
		(T \cdot \widetilde{B})_{nj}&, \text{ if }j\leq d_0  \\
		\mathds{1}_{i = (j-d_0)}&, \text{ else}
	\end{cases}
\end{equation}
It is easy to see that the propagation via these matrices is an overapproximation of the propagation by $\alpha_1$ in the sense that
$\alpha_1\circ\mathfrak{W}_1(\Omega_\rho) \subset \widehat{B}\cdot \widehat{\Omega} + \widehat{b}$.
In the previously mentioned slight abuse of notation, we now close the recursive loop by renaming 
\begin{equation*}
	\Omega_\rho:= \widehat{\Omega}_\rho, \quad d_0:=\widehat{d}_0, \quad \widetilde{B}:=W_2\cdot\widehat{B}, \quad \widetilde{b}= W_2\cdot\widehat{b}+w_2.
\end{equation*}
While recursively continuing this propagation through the entire network, we save the indices and layers of the $\star$-neurons and the layer $\tilde{l} \in \{1,\dots L\}$ in which the first star neuron appears.
We also compute two lists of interval matrices: $[\Lambda]$ and $[\lambda]$. They contain, for each layer, an interval matrix and an interval vector which each contain as their entries an interval defined by the maximum and minimum entries over all possible linear activation states $T^{ni}$ and $t^{ni}$. These interval matrices are necessary for section \ref{sec: lipBounds}.
Note that, unlike in \cite{bhowmick_lipbab_2021}, the interval matrices we obtain in our algorithm are not necessarily diagonal matrices.

\subsection{Lipschitz Bounds}\label{sec: lipBounds}
%\subsection{LinProp}\label{sec: LinProp}
%Assume we have a subproblem $\rho\in\varrho$ with an activation pattern $A$ as computed by SymProp. Let $\tilde{l}$ be the first layer in which $A$ contains a $\star$-neuron. We want to split $\rho$ into sub problems such that the state of this neuron is decided. We do know how to partition the space $\mathbb{R}^{n_{\tilde{l}}}$ into polyhedra in such a way that, if the pre-activation bounds $ze^l$ are constrained to one polyhedron, their activation is linear on that polyhedron. In order to transfer this into half-space constraints on the input space, we use the fact that $A$ has no $\star$-neurons in layers before $\tilde{l}$:
%\begin{align}
%	ze^l &= W^l \cdot xe^{l-1} + b^l = W^l \cdot A^{l-1} \cdot (W^{l-1}\cdot ze^{l-1} + b^{l-1}) + b^l \\
%	&= \underbrace{W^l \cdot A^{l-1} \cdot W^{l-1}\ldots \cdot W^1}_{=:\widetilde{W}} ze^0 + \underbrace{\sum_{i=1}^{l} \left( \prod_{j=i}^{l-1} W^{j+1} \cdot A^{j} \right) \cdot b^i}_{=:\widetilde{b}} ,
%\end{align}
%where we slightly abuse notation to denote by ``$\prod$'' a right-sided matrix-product and set $\prod_{j=l}^{l-1} =:$ Id. Both $\widetilde{W}$ and $\widetilde{b}$ are concrete (i.e. non-interval) matrices. This allows us to easily transfer half-space constraints from $ze^l$ to $ze^0$.

%Needed input:
%\begin{itemize}
%	\item SubProblem $\rho$
%	\item list of weights, biases
%\end{itemize}

To compute the upper Lipschitz bound for an individual subproblem $\rho$ with associated region $\Omega_\rho$, we first assume that the activation state of $f$ on $\Omega_\rho$, in the form of $\widetilde{l}$, as well as $[\Lambda]$ and $[\lambda]$, has already been computed. These computations are either done in SymProp of FFilter.\\
As $\widetilde{l}$ is the first layer index with $\star$-neurons, the function $\widetilde{f}:=(\mathfrak{W}_{\widetilde{l}}\circ\alpha_{\widetilde{l}-1}\circ\dots\circ \mathfrak{W}_1)|_{\Omega_\rho}$ is linear (where we define $\mathfrak{W}_{L+1} =$Id). We can therefore easily compute the Jacobian $J_{\widetilde{f}}$ of $\widetilde{f}$ through a simple matrix product.

If $\widetilde{l}=L+1$, i.e. if the network contains no $\star$-neurons on $\Omega_\rho$, then $f = \widetilde{f}$ and we simply return $||J_{\widetilde{f}}||$. 

If $\widetilde{l}\leq L$, we leverage interval matrix multiplication to construct a matrix which is in each component absolutely larger than $J_f$.
In interval matrix multiplication, the product of two interval matrices $[A], [B$] is again an interval matrix which, as its components, contains intervals that overapproximate all possible results $A\cdot B$, where $A$ and $B$ have entries from the intervals of $[A]$ and $[B]$. This product can be computed through a brute force combinatorial approach which is intuitive but lengthy, which is why we refer to \cite{diep_efficient_2012} for a more comprehensive introduction.

We now define an initial interval matrix $[J_{\widetilde{l}}]:=[J_{\widetilde{f}}, J_{\widetilde{f}}]$ and recursively compute 
\begin{equation}
	[J]_{l+1} = [\Lambda]_l\cdot [W_l, W_l] \cdot [J_l]
\end{equation} 
for $l = \widetilde{l}+1, \dots, L$. Notice that, even though we allow both our linear and activation layers to include biases, their presence does not matter for the computation of the (approximation of the) Jacobian.
Finally, we set $U$ as the $(d_L\times d_0)$-matrix with entries $U_{ij}:= \max(|\underline{[J_{ij}]_L}|, |\overline{[J_{ij}]_L}|)$. 

If now $x \in \Omega_\rho$ is an arbitrary point in which $f$ is differentiable, then there exists a polyhedron $\mathcal{P}\subset \Omega_\rho$ on which $f$ is fixed linear, such that $x\in\mathcal{P}^\circ$. Then, for each layer $l$, the PWL function $\alpha_l$ takes a fixed affine form $A_lx+b_l$ with $A\in[\Lambda]_l$, where we slightly abuse notation to denote by ``$\in$'' a realization of the interval matrix. Then, due to the properties of interval matrix multiplication, 
\begin{equation}
	J_f(x) = A_L\circ\dots \circ W_1 \in [\Lambda]_{L}\circ\dots\circ [W_1,W_1] = [J_{L}].
\end{equation}
Therefore, by definition of $U$, it holds for all $i,j$ that $J_f(x)_{ij}\leq U_{ij}$.

The following lemma therefore proves that $||U||$ is a hard upper bound for $L(f,\Omega_\rho)$:
\begin{lem}[{\cite[Lemma 2]{bhowmick_lipbab_2021}}]\label{lem: bound-lipBabLem}
	Let $U$ be a matrix of the same size as the Jacobian $J_f(x)$ of a PWL function $f$. If $U$ is such that $\sup_{x\in\Omega} |J_f(x)_{ij}| \leq U_{ij}$ for all $i,j$ and all points $x$ in which $f$ is differentiable, then $L_{p\to q}(f, \Omega) \leq ||U||_{p\to q}$.
\end{lem}
In \cite{bhowmick_lipbab_2021}, this lemma is only proven for $p=q$. However, as their proof easily transfers to the general case, we only discuss it in the supplementary material.

\subsection{Branching}\label{sec: methods-branch}

After having described the bounding step, we now describe the branching step of the recursion in more detail. At a given iteration, given the current set of subproblems $\varrho$, we first remove the subproblem $\rho$ with the maximal upper bound from $\varrho$. We assume that $[\Lambda], [\lambda]$ and $\widetilde{l}$ have been computed for $\rho$ in an earlier iteration and that $\rho$ is associated with a polyhedron $\Omega_\rho$. As $\rho$ will also not be solved, we can get the index $n$ of the first $\star$-neuron in layer $\widetilde{l}$.
We then iterate through all polyhedra $\mathcal{P}_i$ in the PWL decomposition of $\alpha_{\widetilde{l}}$ with respect to neuron $n$ as defined in remark \ref{rem: methoids-pwl2}.
Similar to our approach in equation \eqref{eq: symprop-intersectEquiv}, we stack the half-space constraints of $\Omega_\rho$ with the constraints of $\mathcal{P}_i$, propagated back to the input dimension, to obtain a new polyhedron $\mathcal{Q}_i$. This propagation is possible because, on $\rho$, $f$ is fixed linear in all layers before $\widetilde{l}$. If the intersection defined by the stacking of constraints is not empty, we initialize a new subproblem $\rho_i$ associated with $\mathcal{Q}_i$ in input space.
Since $\mathcal{Q}_i\subset \Omega_\rho$, we know that the Lipschitz constant of each new subproblem is not larger than that of $\rho$, meaning that a feasible solution of the new subproblems is also a feasible solution of $\rho$. Also, per definition, each $\rho_i$ has strictly fewer $\star$-neurons than $\rho$ as the state of $n$ is fixed.
We then update the activation states of $\rho_i$ (see below) and compute the upper (lower) bounds. Afterwards, each new subproblem is either terminated according to the rules described in section \ref{sec: methods} or added to $\varrho$, which completes the recursion.

Note that iterating through all polyhedra in a decomposition according to remark \ref{rem: methoids-pwl2} can be intuitively understood as setting the activation state of only that neuron (plus some others that automatically get set with it, depending on the PWL function). One hopes that the constriction enforced by setting one neuron also results in the fixation of others. One could also use a decomposition as in definition \ref{def: methods-pwl} to force this, but this would imply a drastic increase in the number of necessary polyhedra. The advantage would be that we would be certain that layer $\widetilde{l}$ is fixed in all newly created subproblems.

\paragraph{FFilter}
After the branching step, we run a feasibility filter that checks whether we have also implicitly set the state of other neurons. For this, we combine the LinProp and FFilter algorithms of \cite{bhowmick_lipbab_2021}. We multiply the first $\widetilde{l}-1$ layers, on which $f$ is linear into one linear function and then, starting with layer $\widetilde{l}$ check, whether the other layers contain $\star$-neurons or not. Also, $[\Lambda]_{\widetilde{l}}$ and $[\lambda]_{\widetilde{l}}$ are updated. If a layer does not contain $\star$-neurons, its linear state is multiplied with the propagated linear function and $\widetilde{l}$ is increased by one. We terminate this process if we encounter a layer with $\star$-neurons. The activation states of deeper layers are not updated.
\subsection{Conversion guarantee}\label{sec: conversion}

After the algorithm has been run - either to completion or until a fixed approximation quality has been reached - the global upper and lower bounds are returned. If the algorithm is run to completion, these two bounds match and are an exact solution to the maximization problem in equation \eqref{eq: methods-maxOverPoly}: 
\begin{lem}\label{lem: convGuarant}
	The subproblem with the greatest upper bound is an optimal solution of the maximization problem, i.e. its upper bound equals the true Lipschitz constant of $f$.
\end{lem}
This is a standard result from Branch-and-Bound theory \cite{lee_first_2004}. We also mention that our algorithm suffices the theorem of \cite{gendron_parallel_1994}:

\begin{proof}
	As we only implemented an elimination via upper and lower bounds, we only have to check the first two convergence assumptions of \cite{gendron_parallel_1994}.
	Firstly, our algorithm only returns a lower bound for a subproblem $\rho$, if $\rho$ contains no $\star$-neurons. In such a case, the lower bound equals the true Lipschitz constant of $f$ constrained to $\Omega_\rho$ and is therefore a feasible solution to the original problem. Secondly, a subproblem is solved if and only if its upper and lower bound have been computed and are equal. 
\end{proof}
We also mention that, if the initial input space is chosen as $\Omega = \mathbb{R}^{d_0}$, our method computes the true global Lipschitz constant.

\section{Application}
In this section, we apply the ExLipBaB algorithm to networks trained on various datasets and compare the obtained Lipschitz constants to other approximation methods.
However, as a test of basic correctness, we first applied our method to the four ReLU-activated networks for which exact Lipschitz constants were reported in \cite{bhowmick_lipbab_2021}. We are happy to report that our approach is able to replicate the Lipschitz constants computed by their algorithm to over 10 decimals for all tested norms ($||.||_1, ||.||_2, ||.||_{\infty}$). 

In the other listed results, we frequently compare our method to the LipSDP approach, which is a widely cited global approximation approach. For ReLU networks, we use the original implementation \cite{fazlyab_efficient_2019}, if not otherwise mentioned in the most accurate, ``network'' variant. For GroupSort networks with group size two, we use the recently proposed generalization of LipSDP \cite{pauli_novel_2024}, which is based on the more scalable ``neuron'' variant.
All computations were performed under Windows 11 on an AMD Ryzen 9 7900X3D Cpu with 32 GB of RAM. The code for all simulations can be found under \url{https://github.com/tsplittg/ExLipBaB_Code}.  
\begin{table}[t]
	\centering
	\begin{tabular}{lrrrr}
		\toprule
		& \multicolumn{2}{c}{ReLU, 0.11 [0.04]} &  \multicolumn{2}{c}{GroupSort, 0.11 [0.05]}\\
		\cmidrule{2-5}
		& $L(f)$[std] & time (s) & $L(f)$[std] & time (s)\\
		\midrule
		Layerwise &3.12 [0.50]&0.00& 1.62 [0.32] & 0.00 \\
		ExLipBaB &1.15 [0.20]& 0.03& 1.01 [0.24]& 0.03 \\
		LipSDP &1.62 [0.32]& 3.63 &1.30 [0.30]& 0.00 \\
		\bottomrule
	\end{tabular}
	\caption{Test performance and global Lipschitz constant of 2 networks on synthetic data from the absolute value function over 20 runs. Shown are test RMSE mean [standard deviation] of the networks as well as mean Lipschitz estimate [std] for several methods.}
	\label{tab: appl-absoluteVal}
\end{table}

\subsection{Absolute value function}

The absolute value function $|.|$ is of specific importance in the context of Lipschitz-constrained NNs. It has been shown that, despite being 1-Lipschitz, it cannot be learned by a ReLU-NN which is constricted to be Lipschitz continuous with a constant less than 2 by layerwise approximation \cite{huster_limitations_2019}. Other activations, such as GroupSort, do not suffer from this and other theoretical limitation(s) \cite{anil_sorting_2019}.
To put these theoretical results into perspective, we show in table \ref{tab: appl-absoluteVal} the global Lipschitz constants of two unconstrained networks that were trained on data from $|.|$ for 20 random initializations.\\
It can be seen that the ReLU networks, which had learned $|.|$ reasonably well, had a true Lipschitz constant very close to one for most random seeds. However, in the layerwise approximation, their Lipschitz constant was overestimated by almost 200\% on average. For GroupSort networks, the true constant was ``only'' overestimated by, on average, 60\%. It is therefore important to highlight that the aforementioned inapproximability results only speak to the inadequacy of layerwise-constrained ReLU-networks. Still, as restricting a NN's Lipschitz constant usually requires a computation of that constant at every iteration of the optimizer, the computationally efficient layerwise approximation will likely remain one of the most common ways of Lipschitz constant estimation during training. Therefore, the ReLU function remains non-ideal for Lipschitz-constrained NNs.

\subsection{Real data applications}
For realistic data applications, we trained generic, unconstrained NNs on the abalone \cite{warwick_nash_abalone_1994}, wine quality \cite{cortez_modeling_2009} and bike sharing \cite{fanaee-t_event_2014} datasets. We then compare the exactly computed Lipschitz constant with approximations from a layerwise approximation, an estimation by SymProp only and an estimation through LipDSP. For all of these comparisons, we use the $||.||_2$-norm.
To give a meaningful comparison, we also always compute the global Lipschitz constant on $\mathbb{R}^{d_0}$ if not otherwise specified. We discuss in section \ref{sec: appl-regionSize} the implications of different input region sizes for the ExLipBaB algorithm.

As can be seen in the listed results, the bounds approximated by the LipSDP algorithm were relatively tight relative to exact computation in many cases, especially when compared to the layerwise and SymProp approximations. In the relatively simple abalone/ReLU case (table \ref{tab: appl-abalone}), the approximation was almost exact. However, in the moderately more complex small wine/ReLU example (table \ref{tab: appl-wine_small}) and large wine/GroupSort (table \ref{tab: appl-wine_larg}), the approximation by LipSDP was improved by 10\% through exact computation and in the abalone/GroupSort example by almost 50\%.

It is also important to mention that the global computation of an exact Lipschitz constant through ExLipBaB did not converge in two cases (see tables \ref{tab: appl-wine_larg} and \ref{tab: appl-bike}) and the LipSDP algorithm did not (properly) converge in one case (table \ref{tab: appl-bike}). We discuss the impact of the size of the input region in section \ref{sec: appl-regionSize}.

Generally, one can unsurprisingly see that an increase in network size (table \ref{tab: appl-wine_small} vs. table \ref{tab: appl-wine_larg}) results in an increase in runtime. However, dataset complexity seems to be even more impactful (compare table \ref{tab: appl-abalone} vs. table \ref{tab: appl-wine_small}), likely due to a more segmented PWL representation.

\begin{table}[t]
	\centering
	\begin{tabular}{l|rr|rr}
		\toprule
		\multicolumn{5}{l}{Abalone, Shape $[9, 16, 16, 1]$ } \\
		\midrule
		\multicolumn{1}{c|}{\multirow{2}{*}{\LARGE$\mathbb{R}^{d_0}$}}& \multicolumn{2}{c|}{ReLU} & \multicolumn{2}{|c}{MaxMin}\\
		\cmidrule{2-5}
		& estimate & time (s) & estimate & time (s) \\
		\midrule
		Layerwise &14.27&0.000&10.41&0.000 \\
		LipSDP &11.22&30.72&8.827&0.016\\
		SymProp &15.22&0.017&8.918 & 0.015 \\
		ExLipBaB &11.21&20.83&4.487&357.8 \\
		\bottomrule
	\end{tabular}
	\caption{Lipschitz constant estimation of NN trained on the abalone data w.r.t. $||.||_2$}
	\label{tab: appl-abalone}
\end{table}

\begin{table}[t]
	\centering
	\begin{tabular}{l|rr|rr}
		\toprule
		\multicolumn{5}{l}{Wine, Network shape $[11, 12, 12, 1]$ } \\
		\midrule
		\multicolumn{1}{c|}{\multirow{2}{*}{\LARGE$\mathbb{R}^{d_0}$}}& \multicolumn{2}{c|}{ReLU} & \multicolumn{2}{|c}{MaxMin}\\
		\cmidrule{2-5}
		& estimate & time (s) & estimate & time (s) \\
		\midrule
		Layerwise &5.86&0.000&3.78&0.000 \\
		LipSDP (*) &1.84&26.65&2.87&0.014\\
		LipSDP (**) &1.30&4.257&-&-\\
		SymProp &4.78&0.015&11.09& 0.013\\
		ExLipBaB &1.74&36.65&2.79& 2.444\\
		\bottomrule
	\end{tabular}
	\caption{Lipschitz constant estimation of NN trained on the wine quality data w.r.t. $||.||_2$. LipSDP (network) (**) converged to an implausible value; (*) lists the neuron variant also for ReLU.}
	\label{tab: appl-wine_small}
\end{table}
\subsection{Lower bound effectiveness}\label{sec: appl-lowerBound}
\begin{figure}[t]
	\centering
	\includegraphics[width=\columnwidth]{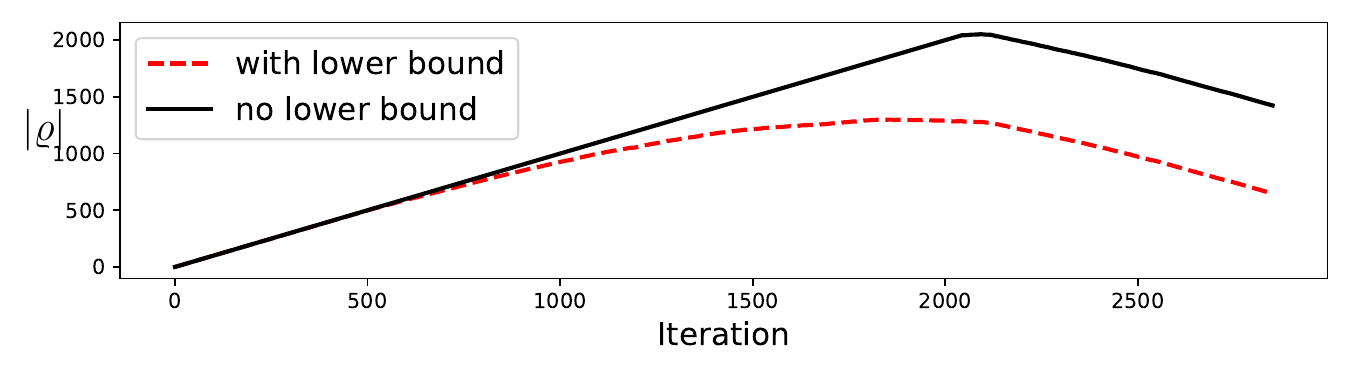}
	\caption{Number of subproblems for ExLipBaB on the NN trained on wine data with and without initial lower bound.}
	\label{fig: effectLowerBnd}
\end{figure}
\begin{table}[t]
	\centering
	\begin{tabular}{l|rr|rr}
		\toprule
		\multicolumn{5}{l}{Wine, Network shape $[11, 24, 24, 1]$  } \\
		\midrule
		\multicolumn{1}{c|}{\multirow{2}{*}{\LARGE$\mathbb{R}^{d_0}$}}& \multicolumn{2}{c|}{ReLU} & \multicolumn{2}{|c}{MaxMin}\\
		\cmidrule{2-5}
		& estimate & time (s) & estimate & time (s) \\
		\midrule
		Layerwise &5.86&0.000&6.10&0.002 \\
		LipSDP &1.19&4.630&4.70&0.012\\
		SymProp &4.78&0.016&25.3&0.040 \\
		ExLipBaB &2.62& (*)&4.23&377.3 \\
		\bottomrule
	\end{tabular}
	\caption{Lipschitz constant estimation of NN trained on Bike sharing data w.r.t. $||.||_2$. (*) The ExLipBaB algorithm did not converge after 12 hours.}
	\label{tab: appl-wine_larg}
\end{table}

\begin{table}[t]
	\centering
	\begin{tabular}{l|rr|rr}
		\toprule
		\multicolumn{5}{l}{Bike sharing, Network shape [13, 20, 16, 8, 3]} \\
		\midrule
		\multicolumn{1}{c|}{\multirow{2}{*}{\LARGE$\mathbb{R}^{d_0}$}}& \multicolumn{2}{c|}{ReLU} & \multicolumn{2}{|c}{MaxMin}\\
		\cmidrule{2-5}
		& estimate & time (s) & estimate & time (s) \\
		\midrule
		Layerwise &30848&0.000&13145&0.000 \\
		LipSDP &5895(*)&5.041&9946&0.018\\
		SymProp &20846&0.043&58484&0.038\\
		ExLipBaB &10802&39283& - & (**) \\
		\bottomrule
	\end{tabular}
	\caption{Lipschitz constant estimation of NN trained on Bike sharing data w.r.t. $||.||_2$. (* The LiSDP algorithm in the forms neuron and layer did not converge; the network variant gave a bound below the true Lipschitz constant computed by ExLipBaB. (**) The algorithm hadn't converged after 24 hours.)}
	\label{tab: appl-bike}
\end{table}

We also test the effect of a simple initial lower bound as proposed in section \ref{sec: symProp}. For this, we test the cardinality of $\varrho$ during a full run of ExLipBaB with and without an initial lower bound computed on $5\cdot 10^6$ grid points. On the wine-network of shape [11, 24, 24, 1] with the GroupSort activation, the lower bound computation took 5.66 seconds, resulting in a bound of 1.629, which resulted in one fewer subproblem being added to $\varrho$. On the network with shape [11, 12, 12, 1] with the ReLU activation, the computation took 2.65 seconds with a lower bound of 1.677. As can be seen in figure \ref{fig: effectLowerBnd}, this resulted in a drastic decrease of $|\varrho|$; however, the number of iterations needed for an exact ExLipBaB run did not change. The usefulness of a lower bound seems therefore highly reliant on its quality and its effect will most likely only lie in a reduction of required memory. Given the comparatively quick computation time of a lower bound, there is also not much downside to using it, especially for low-dimensional inputs, where a good coverage of the input space can be obtained fairly easily.

\subsection{Influence of input space size}\label{sec: appl-regionSize}
Unlike the layerwise and LipSDP estimators, our method is not just able to compute the global Lipschitz constant $L(f, \mathbb{R}^{d_0})$, but also local constants $L(f, \Omega)$ for $\Omega\subsetneq \mathbb{R}^{d_0}$. In addition to reducing the computational cost, due to a smaller number of intersecting linear regions of $f$ with $\Omega$, this can have practical advantages as local robustness guarantees can be far less conservative than global ones.\\
We illustrate this on the network from table \ref{tab: appl-wine_larg}, for which a global computation with the GroupSort activation took 6 minutes with a resulting constant of 4.23. If the input space is restricted to $\Omega = [-0.1, 0.1]^{d_0}$, the computation takes only 5.1 seconds with a constant of 3.32. For $\Omega = [-0.2, 0.2]^{d_0}$, it takes 40.4 seconds for a constant of 3.61 and for $\Omega = [-0.4, 0.4]^{d_0}$ it takes 136 seconds for a constant of 4.12. For $\Omega = [-1, 1]^{d_0}$, the computation is almost equivalent to a global one with a runtime of 365 seconds and a constant of 4.23.

If the input region is chosen small enough, even the Lipschitz constant of a relatively complex network could therefore be feasibly computed in reasonable time.

\section{Conclusion}
The ExLipBaB algorithm we discussed can compute the exact Lipschitz constant for general continuous piecewise linear neural networks. This generalization is of special importance in the field of Lipschitz-constrained NNs, as the ReLU function, which most existing methods are tailored towards, has been shown to have theoretical limitations in this field.
Our generalization also allows a relatively straightforward extension to convolutional- and residual NNs in future research.\\
The importance of Lipschitz constant estimation for convolutional neural networks specifically, even for small ones, has been a topic of study in recent years, often in the context of adversarial robustness.
In the context of residual NNs, Lipschitz constant estimation is often performed in the method of Residual Normalizing Flows. Such Flows need to keep a balance in their allowed Lipschitz constant to be both invertible and expressive. Both fields would therefore benefit from tighter Lipschitz bounds.
However, future research would likely also have to leverage the specific structure and sparsity of CNNs, as Branch-and-Bound methods would likely otherwise not scale to any but the smallest CNNs in practice.\\
Still, our algorithm lends itself naturally to parallelization. In combination with reduced RAM requirements through tighter lower bounds, this could greatly improve its applicability to more complex tabular networks and larger input regions. 

\section*{Acknowledgments}
This work has been funded by the Deutsche Forschungsgemeinschaft (DFG, German Research Foundation) - 459360854 as part of the Research Unit "Lifespan AI: From Longitudinal Data to Lifespan Inference in Health" (DFG FOR 5347), University of Bremen (\url{https://lifespanai.de/})

\clearpage

\bibliographystyle{apalike}
\bibliography{references}

\clearpage

\appendix

\setcounter{section}{0}
\setcounter{page}{1}
\setcounter{figure}{0}
\setcounter{equation}{0}
\setcounter{example}{0}
\setcounter{theorem}{0}
\setcounter{defi}{0}
\setcounter{lem}{0}
\setcounter{rem}{0}

\renewcommand{\thesection}{S\arabic{section}}
\renewcommand{\thepage}{S\arabic{page}}
\renewcommand{\thetable}{S\arabic{table}}
\renewcommand{\thefigure}{S\arabic{figure}}
\renewcommand{\theequation}{S\arabic{equation}}
\renewcommand{\theexample}{S.\arabic{example}}
\renewcommand{\thetheorem}{S.\arabic{theorem}}
\renewcommand{\thedefi}{S.\arabic{defi}}
\renewcommand{\thelem}{S.\arabic{lem}}
\renewcommand{\therem}{S.\arabic{rem}}

\begin{center}
	\LARGE \textbf{{Supplementary Material for \\ ExLipBaB: \\ Exact Lipschitz Constant Computation \\ for  Piecewise Linear Neural Networks}}
\end{center}

\section{Proof of theorem \ref{thm: methods-maxOverPoly}}
In this section, we give more details on the proof of theorem \ref{thm: methods-maxOverPoly}. First, we relate the Lipschitz constant of $f$ to its Jacobian. For differentiable functions, it is well-known that their Lipschitz constant is equal to the supremum over the norm of their Jacobian. As PWL functions generally are not everywhere differentiable, we make use of the generalized Jacobian \cite{jordan_exactly_2020}:
\begin{defi}\label{def: suppl-genJacob}
	The (Clarke) generalized Jacobian of a function $f$ at a point $x$ is defined as the convex hull \\ $\delta_f(x) := \textbf{co}\{\lim_{x_i\to x} J_f(x_i) : x_i \text{ is differentiable}\}$. 
\end{defi}
Informally, the generalized Jacobian can be seen as the convex hull of the Jacobians of points close to $x$ in which $f$ is differentiable.
In \cite{jordan_exactly_2020}, the following theorem is proven for arbitrary Lipschitz-continuous functions and vector norms, which relates the Lipschitz constant of $f$ to the supremum over the generalized Jacobian:
\begin{theorem}\label{thm: methods-jordansThm}
	Let $\alpha:\mathbb{R}^d\to\mathbb{R}^m$ be a (p-q)-Lipschitz continuous function on an open set $\Omega$. Let $\delta_\alpha(\Omega) := \bigcup_{x\in\Omega} \delta_\alpha(x)$, then:
	\begin{equation}\label{eq: methods-LipschJordan}
		L_{p\to q}(\alpha, \Omega) = \sup_{G\in\delta_\alpha(\Omega)} ||G||_{p\to q }.
	\end{equation}
\end{theorem}
Using this theorem, the proof of theorem \ref{thm: methods-maxOverPoly} can be completed:
\begin{proof}
	We use the notation of theorem \ref{thm: methods-maxOverPoly} and definition \ref{def: methods-pwl}. First we argue that $\alpha$ is Lipschitz continuous and that theorem \ref{thm: methods-jordansThm} can be applied. We then only need to prove that $\sup_{G\in\delta_\alpha(\Omega)} ||G||_{p\to q } = \max_{i = 1,\dots, k} ||T_i||_{p\to q}$. For this, let $x\in\Omega$ and $\widetilde{G} \in\delta_\alpha(x)$ be arbitrary. Since $\delta_\alpha(x)$ is a convex set per definition and since all norms are both continuous and convex, Bauer's maximum principle implies that $\max_{G\in\delta_\alpha(x)}$ occurs on some extreme point of $\delta_\alpha(x)$. Since $\delta_\alpha(x)$ is defined as a convex hull (see def. \ref{def: suppl-genJacob}), there must exist a sequence $(x_i)_{i\in\mathbb{N}}$ such that $||\lim_{x_i\to x} J_\alpha(x_i)|| \geq ||\widetilde{G}||$. Since, for an arbitrary index $i$, the vector $x_i$ is in the interior of exactly one polyhedron $\mathcal{P}_{j_i}$ from the PWL decomposition of $\alpha$, it holds that $J_\alpha(x_i) = T_{j_i}$. Therefore, it holds that $||J_\alpha(x_i)|| \leq  \max_{i = 1,\dots, k} ||T_i||$. Since norms are continuous functions, it therefore follows that 
	\begin{equation*}
		||\widetilde{G}|| \leq ||\lim_{x_i\to x} J_\alpha(x_i)|| \leq \max_{i = 1,\dots, k} ||W_i||_{p\to q} \quad .
	\end{equation*}
	Since $x$ and $\widetilde{G}$ were chosen arbitrarily, this proves that $\sup_{G\in\delta_\alpha(\Omega)} ||G||_{p\to q } = \max_{i = 1,\dots, k} ||T_i||_{p\to q}$.
\end{proof}

\section{Proof of equation \eqref{eq: symprop-intersectEquiv}}
In this section, we prove equation \ref{eq: symprop-intersectEquiv} using the notation introduced there.
\begin{proof} 
	$ $ \newline
	\textbf{``$\Rightarrow$'':}
	Let $y\in (W_1 \cdot \Omega_\rho+ w_1)\cap \mathcal{P}_i$. Then, because $y\in (W_1 \cdot \Omega_\rho+ w_1)$, there must exist an $x\in\Omega_\rho$ such that $W_1 \cdot x+w_1 = y$. Also, because $y\in\mathcal{P}_i$, it must fulfill $A_i \cdot y < a_i$. Inserting the first equation into the second, we see that $x$ fulfills $A_i\cdot W_1 \cdot x < a_i - A_i \cdot w_1$. As $x\in\Omega_\rho$, it also holds that $Cx \leq c$, which provers the rightwards direction.\\
	
	\textbf{``$\Leftarrow$'':}
	Let $x \in \{x\in\mathbb{R}^{d_0}: \widetilde{A}_i\cdot x \leq \widetilde{a}_i\}$. Then, per definition of $\widetilde{A}_i$ and $\widetilde{a}_i$, it holds that $C \cdot x \leq c$, meaning $c\in\Omega_\rho$. Also, let $y := \widetilde{B} \cdot x+\widetilde{b}$. Per definition of $x$, it holds that $A_i \cdot \widetilde{B} \cdot x \leq a_i - A_i\cdot \widetilde{b}$, meaning $A_i \cdot y \leq  a_i$, proving the leftwards direction.
\end{proof}

\section{Proof of lemma \ref{lem: bound-lipBabLem}}
As was mentioned in section \ref{sec: lipBounds}, the proof for lemma \ref{lem: bound-lipBabLem} that was given in \cite{bhowmick_lipbab_2021} easily transfers to the general case of $||.||_{p\to q }$-norms, which we show in this section. We therefore explicitly mention that the proofs in this section are only slight variations of the proofs in \cite{bhowmick_lipbab_2021}.
First, we formulate a helper lemma:
\begin{lem}\label{lem: matrixDominanceNorm}
	Let $A$ and $B$ be two $(m\times n)$-dimensional matrices such that componentwise:
	\begin{equation*}
		|A_{ij}|\leq B_{ij} \qquad \forall i, \quad .
	\end{equation*}
	It then holds that $||A||_{p\to q} \leq ||B||_{p\to q}$.
\end{lem}
\begin{proof}
	It is a wel-known result from functional analysis that, because $\{x\in\mathbb{R}^n: ||x||_p=1\}$ is a compact set, there exists an $x^*\in\mathbb{R}^n$ such that $||x||_p=1$ and $||Ax^*||_q = ||A||_{p\to q}$. Define now a vector $x'$ such that $x'_i = |x^*_i|$ for all $i$. From the definition of a $p$ vector norm, one can easily see that $||x'||_p = ||x^*||_p = 1$. Now, the chain of inequalities holds for $1\leq i \leq m$:
	\begin{align*}
		|(Ax^*)_i| &= \big|\sum_{j=1}^n A_{ij} x_j^*\big| \leq \sum_{j=1}^n |A_{ij}| \cdot |x_j^*| \\
		\overset{|A_{ij}|\leq B_{ij}}&{\leq} \sum_{j=1}^n B_{ij} \cdot |x_j^*| = (Bx')_i \quad.
	\end{align*}
	From the definition of the $p$ vector norms, it is now easy to see that the above inequalities imply $||Ax^*||_q\leq ||Bx'||_q$. We can therefore, due to the definition of $x^*$, form the following chain of equalities:
	\begin{equation*}
		||A||_{p\to q} = ||Ax^*||_q \leq ||Bx'||_q \leq ||B||_{p\to q} \cdot ||x'||_p = ||B||_{p\to q}
	\end{equation*}
\end{proof}
Having proven lemma \ref{lem: matrixDominanceNorm}, the proof of lemma \ref{lem: bound-lipBabLem} is almost finished.
\begin{proof}[Proof of lemma \ref{lem: bound-lipBabLem}]
	For differentiable functions, it is well known that $L_{p\to q}(f,\Omega) = \sup_{x\in\Omega} ||J_f(x)||_{p\to q}$, which together with lemma \ref{lem: matrixDominanceNorm} would prove the result. However, we assume a PWL (and not necessarily differentiable) $f$. In \cite{bhowmick_lipbab_2021}, the authors cite \cite{jordan_exactly_2020}, to get a similar result for not everywhere differentiable functions. To our knowledge, \cite{jordan_exactly_2020} however only prove such a result in the case of general position ReLU networks. As we only require a proof for PWL functions $f$, we refer therefore to theorem \ref{thm: methods-maxOverPoly}, which together with lemma \ref{lem: matrixDominanceNorm} also proves lemma \ref{lem: bound-lipBabLem}.
\end{proof}

\section{Details on practical applications}
\subsection{Bike sharing}
For the bike sharing dataset, we used a $0.64-0.16-0.2$ train-validation-test split of the data. The ``date'' feature was converted into a ``YYYYMMDD`''-coded integer. The data was also normalized. We trained ReLU and GroupSort networks of shape $(13, 20, 16, 8, 3)$ with the Adam optimizer with a learning rate of 0.001 and weights loss of $10^{-3}$. We trained for a maximum of 2000 epochs with a batch size of 256 and performed early stopping with a patience of 5 and a minimum improvement of $10^{6}$.\\
The ReLU network achieved a test RMSE of 40.33. The GroupSort network with group size 2 achieved a test RMSE of 40.38 and that with group size 4 one of 37.58.

Net GroupSort with shape 13x18x18x18x3, groupsize 2 had RMSE 37.94

Network with width 24 had 35.97

\subsection{Wine quality dataset}
For the wine quality dataset, we again used a $0.64-0.16-0.2$ train-validation-test split of the data. We used 11 features, deleting the ``color'' feature, using the ``quality'' as target. We trained two networks, both used a learning rate of 0.001, a batch size of 256 and 1000 epochs with the Adam optimizer with a patience of 5 and a minimum improvement of $10^{6}$.\\
The network of shape [11, 12, 12, 1] achieved a test RMSE of 0.703 with the ReLU activation and 0.712 with the GroupSort activation with group size 2. The network of size $[11, 24, 24, 1]$ achieved a test RMSE of 0.712 for the ReLU- and 0.692 for the GroupSort activation.

\subsection{Abalone}
For the abalone dataset, we again used a $0.64-0.16-0.2$ train-validation-test split of the data. The categorical ``Sex'' feature was one-hot encoded and the reference dropped. The data was also normalized.\\
We then trained ReLU, LeakyReLU and GroupSort networks of shape $(9, 16, 16, 1)$ with the Adam optimizer with a learning rate of 0.001. We trained for a maximum of 1000 epochs with a batch size of 64 and performed early stopping with a patience of 5 and a minimum improvement of $10^{6}$.\\
The ReLU network achieved a test RMSE of 2.14, the GroupSort network one of 2.14 with a group size of two (MaxMin activation) and one of 2.17 for group size four.

%\subsection{Iris}
%For the iris data, we again used a $0.64-0.16-0.2$ train-validation-test split. The target variable was one-hot encoded with three levels. The independent features were normalized.\\
%We then trained ReLU, LeakyReLU and GroupSort networks of shape $(4, 6, 6, 3)$ with the Adam optimizer with a learning rate of 0.001. We trained for a maximum of 1000 epochs with a batch size of 64 and performed early stopping with a patience of 5 and a minimum improvement of $10^{6}$.As loss function, we use the combined softmax and BCE loss, which, in pytorch is called BCEWithLogitsLoss.\\
%The ReLU network achieved a test accuracy of 0.97 (BCE 0.072), the GroupSort network an accuracy of 0.97 (BCE 0.074) with a group size of two (MaxMin activation) and 0.97 (BCE 0.068) for group size three.

\section{Common activation Functions}\label{sec: suppl-actfctns}
In this section, we discuss results pertaining to commonly used activation functions and also describe the polyhedron decomposition we use in our implementation of these functions.

\subsection{Componentwise PWL functions}\label{sec: suppl-actfctns-compWise}
In this subsection, we show that an activation function that applies a componentwise one-dimensional linear spline, such as ReLU, to the pre-activation is a PWL function in the sense of remark \ref{rem: methoids-pwl2}. To keep notation simple, we will only discuss settings in which the spline is the same on all dimensions, but our treatment can be easily transferred to the general case.\\
Let $\alpha:\mathbb{R}\to\mathbb{R}, x\mapsto \sum_{j=1}^k (a_j x + b_j) \mathds{1}_{I_j}(x)$, where $I_j$ are disjunct intervals whose union is $\mathbb{R}$ and where $a_j$ and $b_j$ are appropriate coefficients. It can be easily seen, that this structure includes ReLU for $a_1 = 0, b_1 = 0, I_1 =(-\infty, 0] $ and $a_1 = 1, b_1 = 0, I_1 = (0, \infty) $. It also includes LeakyReLU, simple learned activations like ParametricReLU and complex ones like learned spline activations.\\
We now assume that $\widetilde{\alpha}:\mathbb{R}^d\to\mathbb{R}^d$ is a componentwise activation function that applies $\alpha$ to each component of the pre-activation. Let then for $i=1,\dots, d$ and $j=1,\dots, k$ be $\mathcal{P}_{ij}$ be the polyhedron defined by:
\begin{equation*}
	\mathcal{P}_{ij} = \mathbb{R}^{i-1} \times I_j \times \mathbb{R}^{d-i} \quad ,
\end{equation*}

where we omit $\mathbb{R}^0$ when forming the product. We now claim that the seet $\{\mathcal{P}_{ij}\}_{\substack{i=1,\dots, d\\ j=1,\dots, k}}$ is a decomposition as in remark \ref{rem: methoids-pwl2}. This is easy to see, as, for an arbitrary neuron $n$, the union $\bigcup_{j=1, \dots, k} \mathcal{P}_{nj}$ is equal to $\mathbb{R}^d$ and on each $\mathcal{P}_{nj}$, the state of neuron $n$ is fixed, where $T^{nj}$ is a zero-vector except for the entry $a_j$ at its $n$-th position and $t^{nj}=b_j$.\\
This is the generalization of the hyper-rectangle representation that \cite{bhowmick_lipbab_2021} used.

\subsection{GroupSort}
In this subsection, we construct the polyhedron decomposition of the GroupSort activation function that was introduced in \cite{anil_sorting_2019} and further studied e.g. in \cite{tanielian_approximating_2021-1}. A GroupSort layer with group size $\gamma$ will split the pre-activation into groups of size $\gamma$, sort the entries of each group in ascending order and return the results.\\

We also prove mathematically that this is a PWL function and construct an appropriate decomposition of our method. A study of the relation between GroupSort and other PWL functions can also be found in \cite{tanielian_approximating_2021-1}.\\
Our treatment naturally caries over to MaxMin and Fullsort, which are special cases of GroupSort with either one group only or $g=2$ as \cite{anil_sorting_2019} mentions.\\
As a starting point, we discuss the FullSort activation, as the GroupSort activation is essentially just composed of multiple FullSort functions of lower dimensionality.
Let therefore $\alpha:\mathbb{R}^\gamma \to\mathbb{R}^\gamma$ be the FullSort activation, i.e. the function that transforms any vector $(x_1,\dots, x_\gamma)$ into a vector with all entries sorted in ascending order:
\begin{equation}
	\alpha((x_1,\dots, x_\gamma)) = (x_{i_1}, \dots, x_{i_\gamma}) ,  
\end{equation}
with $x_{i_1}\leq x_{i_2} \leq \dots \leq x_{i_\gamma}$.

\begin{lem}\label{lem: suppl-actfctns-fullsortPWL}
	FullSort is a PWL function.
\end{lem}

\begin{proof}
	Let $\pi = (j_1, \dots, j_d)$ be an arbitrary ordering of ${1, \dots, \gamma}$ (i.e. a \emph{permutation}) and let $A$ be the $(\gamma-1)\times \gamma$-matrix defined by:
	\begin{equation}
		A_{kl} = \begin{cases}
			1 &, \text{if } l=j_k \\ 
			-1 &, \text{if } l=j_{k+1} \\
			0 &, \text{else}
		\end{cases}
	\end{equation}
	and let $b$ be a zero-vector of length $\gamma-1$. Then the following holds:
	\begin{equation*}
		\{x\in\mathbb{R}^\gamma: x_{j_1}\leq \dots\leq x_{j_\gamma}\} = \{x\in\mathbb{R}^\gamma: Ax\leq b\} =:\mathcal{P}_\pi
	\end{equation*}
	and it is easy to see that we can decompose $\mathbb{R}^\gamma$ into such polyhedra.
	For any $x\in \mathcal{P}_\pi$, we can now compute $\alpha$:
	\begin{equation}\label{eq: suppl- actfctns-groupViaFull}
		\alpha(x)= (x_{j_1}, \dots, x_{j_\gamma}),
	\end{equation}
	
	which can be equivalently described as a transformation with a $(\gamma\times\gamma)$-permutation matrix $P$ given by:
	\begin{equation}
		P_{kl} = \begin{cases}
			1 &, \text{if } l = j_k \\
			0 &, \text{else}
		\end{cases}.
	\end{equation}
	We have therefore constructed a decomposition of $\mathbb{R}^\gamma$ into $(\gamma!)$-many polyhedra and, constrained to each polyhedron, $\alpha$ is a linear function.
\end{proof}

Let now $\tilde{\alpha}: \mathbb{R}^d\to \mathbb{R}^d$ be the GroupSort activation function with group size $\gamma$, which splits the pre-activation into $k$ different groups of size $\gamma$, sorts the elements of each group in ascending order and combines the result as its output. We assume that $\gamma$ is a divider of $d$, i.e. that $k\cdot \gamma = d$, but our discussion in this chapter as well as our implementation would easily translate to groups of inhomogeneous size.\\
Using earlier notation, let $\alpha$ again be the FullSort activation operating on $\gamma$-dimensional space. For any neuron with index $i\leq d$, we can decompose the index as $i = l_1\cdot\gamma+l_2$ and write:
\begin{equation}
	\tilde{\alpha}_i = \alpha((x_{l_1\gamma+1},\dots, x_{(l_1+1)\gamma}))_{l_2}, 
\end{equation}
which formalizes the intuition that GroupSort is essentially just a combination of multiple lower-dimensional FullSort functions.\\
It is therefore not difficult to see that one can construct a decomposition of a form as in definition \ref{def: methods-pwl} of $\mathbb{R}^d$ into polyhedra by forming all possible combinations of $k$ polyhedra from decompositions of $\mathbb{R}^\gamma$. Similarly, as can be seen from equation \eqref{eq: suppl- actfctns-groupViaFull}, in any such polyhedron, the state of all $d$ neurons would be fixed linear.\\
Similar to section \ref{sec: suppl-actfctns-compWise}, we however instead construct a decomposition following remark \ref{rem: methoids-pwl2} in which only the activation state of a subset of neurons is fixed in each polyhedron. For this, let $\mathcal{P}_1, \dots, \mathcal{P}_N$ be a decomposition of $\mathbb{R}^\gamma$ into polyhedra, such as the one constructed in the proof of lemma \ref{lem: suppl-actfctns-fullsortPWL}. We then define the following set of polyhedra:
\begin{equation*}
	\mathcal{Q}_{ij} := \mathbb{R}^{(i-1)\gamma} \times \mathcal{P}_j \times \mathbb{R}^{(k-i)\gamma} 
\end{equation*}
for $i=1,\dots, k$ and $j= 1,\dots,N$, where $\mathbb{R}^0$ is ignored when forming the cartesian product.\\
Then each $\mathcal{Q}_{ij}$ is a polyhedron of dimensionality $\gamma \cdot k =d$ and $\bigcup_{\substack{i\leq k\\j\leq N}} \mathcal{Q}_{ij} = \mathbb{R}^d$. \\
Also, on each $\mathcal{Q}_{ij}$, the state of the neurons with indices $(i-1)\gamma+1,\dots, i\gamma$ are fixed with their linear activation states given by the one of FullSort on $\mathcal{P}_j$.

\section{Code Listings}
In this section, we show pseudocode for the branch function as well as a more detailed pseudocode of the main algorithm than the on eprovideed in listing \ref{alg: simplifiedMain}.
\begin{algorithm}
	\caption{Branch}
	\hspace*{\algorithmicindent} \textbf{Input:}
	\begin{itemize}
		\item sub-problem $\rho$
		\item old glb
		\item list of weights and biases of Neural net
		\item $\alpha$ list of instances of PWL
		\item pnorm
	\end{itemize}
	\hspace*{\algorithmicindent} \textbf{Output: }
	\begin{itemize}
		\item new\_subproblems: list of instances of SubProblem class
		\item glb: float, new lower bound on Lipschitz constant over all sub-problems
	\end{itemize}	
	\begin{algorithmic}[1]
		\STATE initialize: new\_subproblems $\leftarrow$ empty list
		\STATE $\star$-neuron\_indices $\leftarrow$ $\rho.\star$-neurons$[\rho.\tilde{l}]$
		\STATE neuron\_to\_branch $\leftarrow$ $\star$-neuron\_indices$[0]$ //branch at first $\star$-neuron 
		\STATE possible\_polyhedra $\leftarrow$ $\alpha[\rho_{\tilde{l}}].$get\_polyhedron\_ decompostition(neuron\_to\_branch)
		
		\FOR{polyhedron in possible\_polyhedra}
		\STATE intersection $\leftarrow $ join constraints of $\Omega_\rho$ and polyhedron backwards propagated with $\rho$.propagated\_weights
		\IF{intersection is empty}
		\STATE go to next polyhedron
		\ENDIF
		\STATE act\_state\_$[\Lambda]$ , act\_state\_$[\lambda]$ $\leftarrow$ $\rho.[\Lambda], \rho.[\lambda]$
		\FOR{known decided\_neuron \textbf{in} polyhedron}
		\STATE update layer $\tilde{l}$ of act\_state\_$[\Lambda]$ , act\_state\_$[\lambda]$
		\ENDFOR
		\STATE create subproblem $\rho_{new} \leftarrow $ SubProblem(polyhedron, act\_state\_$[\Lambda]$ , act\_state\_$[\lambda]$)
		\STATE $\rho_{new}.\tilde{l} \leftarrow \rho.\tilde{l}$ 
		\STATE $\rho_{new}.$star\_neurons $\leftarrow$ $\rho.$star\_neurons $\backslash$ known decided\_neurons 
		\STATE $\rho_{new}$ $\leftarrow$ $\rho_{new}$.FFilterProp(weights, $\alpha$) //update the activation pattern and get $\tilde{l}$
		\STATE $L_{ub} \leftarrow$ LipschitzBounds($\rho_{new}$, $( \widetilde{W}, \widetilde{b}))$ //see section \ref{sec: lipBounds}
		\STATE new\_subproblems $\leftarrow$ new\_subproblems.append($\rho_{new}$)
		\IF{$\rho_{new}.\tilde{l}==L+1$}
		\STATE //no $\star$-neurons, one linear region
		\STATE glb $\leftarrow \max(\text{glb}, L_{ub}$)				
		\ENDIF
		\ENDFOR 
		\RETURN{new\_subproblems, glb} 
	\end{algorithmic}
	\label{alg: branch}	
\end{algorithm}

\begin{algorithm}
	\caption{Main Algorithm}
	\hspace*{\algorithmicindent} \textbf{Input:}
	\begin{itemize}
		\item Network $f = \alpha^L\circ \mathfrak{W}^L\circ\dots\alpha^1\circ\mathfrak{W}^1 $ (list of alternating tuples ($W^l$,$w^l$) and instances of PWL $\alpha^l$)
		\item input domain $\Omega$ (polyhedron)
		\item approximation factor $k$ (float)
		\item $p$, deciding which $p$-norm to use
	\end{itemize}
	\hspace*{\algorithmicindent} \textbf{Output: }
	\begin{itemize}
		\item (glb, gub) tuple of greatest lower bound and greatest upper bound for $L(f, \Omega)$ where gub $\leq k\cdot$ glb
	\end{itemize}	
	\begin{algorithmic}[1]
		\STATE initialize constraints $H \leftarrow \Omega \subset\mathbb{R}^{n_0}$ //(polyhedron)
		\STATE weights\_list, act\_list $\leftarrow f$ //decompose $f$ into weights and activations lists 
		\STATE initial activation pattern ($[\Lambda]$, $[\lambda]$, $\tilde{l}$, star\_neurons) $\leftarrow $ SymProp($\Omega$, weights\_list, act\_list)
		\STATE $\rho_I \leftarrow$ SubProblem($H, [\Lambda],  [\lambda]$)  // instance of SubProblem class
		\STATE $\rho_I.\tilde{l} \leftarrow \tilde{l}$
		\STATE $\rho_I.$star\_neurons $\leftarrow$ star\_neurons
		\STATE $\rho_I.$propagated\_weights $\leftarrow$ LinProp(Id, $\textbf{0}$, weights\_list, $[\Lambda],  [\lambda]$, start=0, end =$\rho_I.\tilde{l}$)
		\STATE $\rho_I.L_{ub} \leftarrow$ LipschitzBounds($\rho_I$, weights\_list, $p$) //see seection \ref{sec: lipBounds}
		\STATE gub, glb $\leftarrow (\rho_I.L_{ub}, 0)$ // initialize upper and lower bound on Lipschitz constant
		\STATE initialize: $\varrho \leftarrow $ maxHeap($[\rho_I]$) //(initialize maxheap of subproblems)
		\IF{$\tilde{l}$ == L}
		\STATE //trivial case
		\STATE glb $\leftarrow \rho_I.L_{ub}$
		\ENDIF
		\WHILE{gub $> k\cdot $ glb}
		\STATE $\rho' \leftarrow \arg\max_{\rho\in\varrho}(\rho.L_{ub})$
		\STATE $\varrho \leftarrow \varrho\backslash\{\rho'\}$
		\STATE (new\_subproblems, glb) $\leftarrow$ Branch($\rho'$, glb, weights\_list, act\_list, pnorm)
		\FOR{$\tilde{\rho}$ in new\_subproblems}
		\IF{$\tilde{\rho}.L_{ub}>glb$}
		\STATE $\varrho \leftarrow (\varrho \; \cup \{\tilde{\rho}\} )$
		\ENDIF
		\ENDFOR
		\STATE gub $\leftarrow \max_{\rho\in\varrho} \rho.L_{ub}$
		\ENDWHILE
		\RETURN (glb, gub)
	\end{algorithmic}
	\label{alg: main}	
\end{algorithm}

\end{document}